\documentclass[brazil,english, conference, american, harvard]{sbatex}
\usepackage[utf8]{inputenc}
\usepackage[T1]{fontenc}
\makeatletter
\def\verbatim@font{\normalfont\ttfamily\footnotesize}
\makeatother
\usepackage{amsmath}
\usepackage{graphicx}
\usepackage{float}
\usepackage{amssymb}
\usepackage{subfigure}
\usepackage{ae}
\usepackage{epstopdf}
\usepackage{algpseudocode}
\usepackage{algorithm}
\usepackage{color}
\usepackage[hidelinks]{hyperref}

\newcommand{\blfootnote}[1]{%
  \begingroup
  \renewcommand{\thefootnote}{}%
  \footnotetext{#1}%
  \addtocounter{footnote}{-1}%
  \endgroup
}

\begin{document}

\title{AN OPTIMAL ALGORITHM FOR CHANGING FROM LATITUDINAL TO LONGITUDINAL FORMATION OF AUTONOMOUS AIRCRAFT SQUADRONS}

\author[1]{Paulo André S. Giacomin,}
          {pasgiacomin@uesc.br, }

\address{Universidade Estadual de Santa Cruz\\
Departamento de Ciências Exatas e Tecnológicas - Área de Informática\\
Ilhéus-BA, Brasil}

\address{Instituto Tecnológico de Aeronáutica\\
Departamento de Sistemas e Controle - Divisão de Engenharia Eletrônica \\
São José dos Campos-SP, Brasil}

\author[2]{Elder M. Hemerly}
          {hemerly@ita.br}

\twocolumn[

\maketitle

\selectlanguage{english}

\begin{abstract}
This work presents an algorithm for changing from latitudinal to 
longitudinal formation of autonomous aircraft squadrons. The maneuvers are defined  dynamically by using a predefined set of 3D basic maneuvers. This formation change is necessary when the squadron has to perform tasks which demand both formations, such as  lift-off,  georeferencing, obstacle avoidance and landing. Simulations  show that the formation change is done without collision. The time complexity analysis of the transformation algorithm reveals that its efficiency is optimal, and the proof of correctness ensures its longitudinal formation features.
\end{abstract}


\selectlanguage{brazil}

\begin{abstract}
Este trabalho apresenta um algoritmo de mudança de formação em latitude para formação longitudinal de esquadrilhas de aeronaves autônomas. As manobras são definidas dinamicamente utilizando-se um conjunto pré-definido de manobras 3D básicas. Esta mudança de formação é necessária quando a esquadrilha tem que desenvolver tarefas que demandam ambas as formações, tais como decolagem, georreferenciamento, desvio de obstáculos e aterrissagem. As simulações mostram que a mudança de formação é feita sem colisão. A análise de complexidade de tempo do algoritmo de transformação revela que sua eficiência é ótima, e a prova de correção  assegura suas características de formação longitudinal.
\end{abstract}

]
\selectlanguage{english}

\section{Introduction}

Recently, it has been possible to see a growing interest in the development of autonomous 
aircraft that can cooperate with police and other organizations in the solution of 
public security problems. The basic motivations are:  autonomous  agents can deal with
dangerous or health-threatening problems, like fires, violence monitoring, inspection of
nuclear areas, deforestation monitoring and monitoring of areas with armed conflict, without
exposing humans to the risks.

\blfootnote{ Published in the XI Brazilian Symposium on Intelligent Automation, 2013; the scientific content remains unchanged. \href{https://doi.org/10.5281/zenodo.21460781}{DOI: 10.5281/zenodo.21460781} Licence: CC-BY-NC-ND. }

When several agents are used in the solution of these problems, some advantages arise:
a) distributed systems are usually more robust  than  centralized ones, and b) it is
possible to make better use of sensors, since they can be shared by the network.
As an additional example, when autonomous aircraft squadrons are used in georeferencing, 
the visual field of the cameras increases, as shown in Figure \ref{fig:georeferenciamento}, and the
captured images can be mosaicked. Besides, autonomous aircraft typically
fly at low heights; hence, good-quality images can be captured.

\begin{figure}
 \centering
 \includegraphics[scale=1]{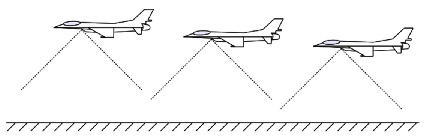}
 \caption{The visual field of the aircraft is increased by using squadrons, enabling larger mosaics.}
 \label{fig:georeferenciamento}
\end{figure}

The latitudinal formation presented in Figure \ref{fig:latitudinal-formation}
is an attractive formation to deal with the georeferencing problem,
since better area coverage can be achieved. However, obstacles can appear
during the flight, and it may be necessary to change the squadron formation to avoid them. 
For example, the squadron can assume the longitudinal formation presented in Figure \ref{fig:longitudinal-formation} for
collision avoidance. Thus, if the first aircraft succeeds in avoiding the collision, all other
aircraft in the squadron can also avoid the obstacle by using the same behavior. 

\begin{figure}[h]
 \centering
  ~~~~~~~~\includegraphics[bb=14 14 227 103,scale=1]{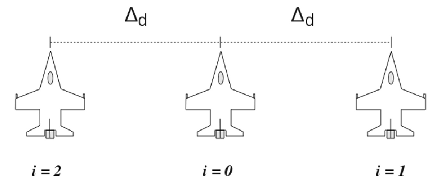}
 \caption{The latitudinal formation. It is possible to achieve good area coverage with it.}
 \label{fig:latitudinal-formation}
\end{figure}

\begin{figure}[h]
 \centering
 ~~~~~~~~\includegraphics[bb=14 14 244 70,scale=0.9]{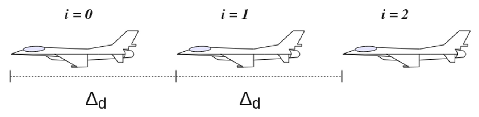}
 \caption{The longitudinal formation. It can be used for collision avoidance, landing, and lift-off.}
 \label{fig:longitudinal-formation}
\end{figure}

The longitudinal formation is also necessary when the squadron
is landing and taking off. Therefore, if the same squadron needs to perform
georeferencing, landing,  lift-off  and obstacle avoidance, it will eventually be necessary
for the squadron to change between its latitudinal and longitudinal formations. 

Therefore, the problem considered in this work is: assuming that 
there are $n$ aircraft in latitudinal formation, equally spaced by 
$\Delta_d$ meters, we want to develop an algorithm that changes 
the squadron to the longitudinal formation, where the aircraft 
will also be equally spaced by $\Delta_d$ meters, without collision among them.

This problem involves aircraft formation, 
trajectory generation and control. The nonlinear model predictive control approach
 of \cite{Chao} and the leader-follower approach of \cite{You} and \cite{Gu} 
deal with the formation problem, but changing between well-defined geometric formations
is not considered. Trajectory generation made by using optimization algorithms is usually found in the literature. Some examples are the works of
\cite{Cheng} and \cite{Xu}, but works like these do not 
focus on well-defined geometric formations. Other techniques are also used in the trajectory
definition, like the geometric moments, controlled via nonlinear gradient \cite{Morbidi}, 
the modern matrix analysis used by \cite{Coker}, the combination of the hybrid navigation 
architecture with the local obstacle avoidance methodology and with the model
predictive control \cite{Jansen}, the navigation functions \cite{Roussos}, and the
variation of rapidly-exploring random trees \cite{Neto}. But
they also do not consider the changes between well-defined geometric formations during
the flight. Control techniques, like reinforcement learning \cite{Santos} do not
focus on the change of formation.

The formation reconfiguration is 
studied by \cite{Venkataramanan}, where the aircraft move their position 
inside a formation, and the same formation is considered before and after 
the reconfiguration. The autonomous decision-making architecture \cite{Knoll}
also considers this problem, but neither \cite{Venkataramanan} nor  \cite{Knoll} considers
the transition between different formations. 
After an exhaustive literature search,  no algorithm was found dealing  with the problem considered here.
Thus, the main contributions of this work are:

\begin{enumerate}

 \item An algorithm for changing from the latitudinal to the longitudinal
formation of the squadron. The time complexity analysis of the proposed algorithm shows its 
efficiency is optimal. 

\item A proof of correctness of the proposed algorithm, which ensures its longitudinal formation features.

\item  The simulation results by considering a case study, in which 
the aircraft do not collide.

\end{enumerate}

These contributions are described in the next sections, 
namely: \ref{sec:methodology} - Methodology, \ref{sec:simulation} - Simulation Results, 
\ref{sec:theoretical-analysis} - Theoretical Analysis and \ref{sec:conclusion} - Conclusion. 

\section{Methodology}
\label{sec:methodology}

The proposed algorithm needs to create references to be followed by the aircraft.
To do this, a set of maneuvers is specified. 

\subsection{The Maneuver Schemes}

The proposed algorithm employs two 3D basic maneuver schemes, 
as shown in Figures \ref{fig:go-forward} and \ref{fig:to-turn}.
For details, see \cite{Giacomin-cobem}.
Algorithms are used to create the references by using the specifications presented 
in these figures. Here they are called
FW and C to implement the  \textit{go-forward}  and  \textit{to-turn}  maneuvers, respectively.
The  \textit{go-forward}  interface is FW($P$, $\varphi$, $\beta$, $vel$, $d$, $T$)
and the  \textit{to-turn}  interface is C($P$, $\varphi$, $\beta$, $\theta$, $\alpha$, $vel$, $r$, $T$),
where $P$ is the initial position $(x_0, y_0, h_0)$, $T$ is a time vector to be filled
with time intervals, and the other parameters are shown in 
Figures \ref{fig:go-forward} and \ref{fig:to-turn}. These algorithms are used by the transition algorithm described in Subsection \ref{subsection:the-transition-algorithm}.

\begin{figure}[h]
 \centering
  \includegraphics[bb=0 0 158 83,scale=1]{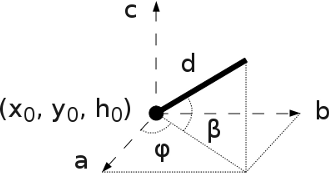}
  \caption{ \textit{Go-forward}  maneuver: two angles, a point and a distance are specified, with $\vec{a} \parallel \vec{x}$, $\vec{b} \parallel \vec{y}$ and $\vec{c} \parallel \vec{h}$.}
  \label{fig:go-forward}
\end{figure}

\begin{figure}[h]
 \centering
 \subfigure[Part a]{\includegraphics[bb=0 0 121 83,scale=0.8]{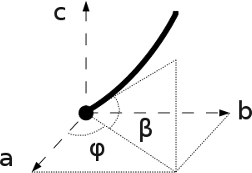}}
 \subfigure[Part b]{\includegraphics[bb=0 0 118 83,scale=0.8]{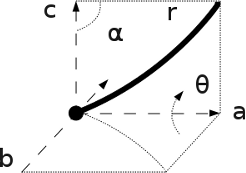}}
  \caption{ \textit{To-turn}  maneuver: four angles, a 3D point,  and  a radius need to be specified.}
  \label{fig:to-turn}
\end{figure}

\subsection{The Transition Algorithm}
\label{subsection:the-transition-algorithm}

The transition algorithm is shown in Algorithm \ref{algo:fltl}. It is called 
FLATLO due to the initial characters that describe its function:
From Latitudinal To Longitudinal formation. It basically performs 
the maneuvers presented in Figure \ref{fig:geometria-trajetoria} by each 
aircraft on the left of the squadron, and the equivalent mirrored  maneuver  for each aircraft on the right of the squadron.

\begin{algorithm}
  
  \caption{The transition algorithm}\label{algo:fltl}
  
  \begin{algorithmic}[1]  
   \Procedure{FLATLO}{$i$, $\Delta_d$, $vel$, $r$, $\varphi$, $\beta$, $P$}

    \State $T \gets \varnothing$ 
    \State $Ref \gets \varnothing$ 

    \State $\Delta_{m} \gets move\_forward(i, \Delta_d)$ \label{linha:deslocamento}

    \State $side \gets get\_side(i, N)$

    \State $a \gets FW(P, \varphi, \beta, vel, \Delta_m, T)$ \label{linha:primeiro_movimento}
    \State $Ref \gets Ref \cup a$

    \If{$|Ref| \neq 0$}
      \State $P_{1} \gets last(Ref)$
    \Else
      \State $P_{1} \gets P$
    \EndIf

    \State $k \gets 0$

    \While {$k < 4$}

      \If{$|Ref| \neq 0$} \label{linha:inicio_encaixe}
	\State $P_{1} \gets last(Ref)$ 
	\State $P_{0} \gets penultimate(Ref)$
	\State $D \gets distance(P_1, P_0)$
	\State $\varphi_2 \gets sin^{-1}(\Delta_y / D)$
	\State $\beta_2 \gets sin^{-1}(\Delta_h / D)$ 

      \Else
	\State $\varphi_2 = \varphi$
        \State $\beta_2 = \beta$
      \EndIf \label{linha:final_encaixe}

      \If{$k = 0$} 
	\If{$side = LEFT$} \label{linha:inicio_primeira_curva}
	  \State $a \gets C(P_1, \varphi_2, \beta_2, \pi, \frac{\pi}{4}, vel, r, T)$
	\Else
	  \State $a \gets C(P_1, \varphi_2, \beta_2, 0, \frac{\pi}{4}, vel, r, T)$
	\EndIf \label{linha:final_primeira_curva}
      \ElsIf{$k = 1$}
	\If{$even(N)$} \label{linha:delta_y_comeco}
	  \State $\Delta \gets \Delta_d / 2 + \lfloor i/ 2 \rfloor \cdot \Delta_d$
	\Else
	    \If{$i \neq 0 $}
	      \State $\Delta \gets \Delta_d \cdot \lfloor (i+1) / 2 \rfloor$
	    \Else
	      \State $\Delta \gets 0$
	    \EndIf
	\EndIf \label{linha:delta_y_final}

	\State $\Delta_2 \gets \sqrt{2} \cdot (\Delta - 2 \cdot r \cdot (1 - cos(\pi / 4)))$
	\State $a \gets FW(P_1, \varphi_2, \beta_2, vel, \Delta_2, T)$ \label{linha:vai_em_direcao_ao_centro}
      \ElsIf{$k = 2$}

	\If{$side = LEFT$} \label{linha:inicio_segunda_curva}
	  \State $a \gets C(P_1, \varphi_2, \beta_2, 0, \frac{\pi}{4}, vel, r, T)$
	\Else
	  \State $\varphi_3 \gets - \varphi_2$
	  \State $a \gets C(P_1, \varphi_3, \beta_2, \pi, \frac{\pi}{4}, vel, r, T)$
	\EndIf \label{linha:final_segunda_curva}

      \Else
	\State $a \gets FW(P_1, \varphi_2, \beta_2, vel, K_2, T)$ \label{linha:manobra_final}
      \EndIf

    \State $Ref \gets Ref \cup a$
    \State $k \gets k + 1$
    \EndWhile

    \State \textbf{return} $Ref$ 
   \EndProcedure
  \end{algorithmic} 
\end{algorithm}

\begin{algorithm}
  
  \caption{The advancing algorithm}\label{algo:advancing}
  
  \begin{algorithmic}[1]  
   \Procedure{move\_forward}{$i$, $\Delta_d$}

    \If{$odd(N)$}
      \If{$odd(i)$}
	\State $k \gets \lfloor (N - i) / 2 \rfloor$ 
	\State $\Delta_m \gets k \cdot \Delta_d + 2 \cdot (k-1) \cdot \Delta_d$
      \Else
	\State $\Delta_m \gets 3 \cdot \Delta_d \cdot \lfloor (N - i - 1) / 2 \rfloor $
      \EndIf
    \Else
      \If{$odd(i)$}
	\State $\Delta_m \gets 3 \cdot \Delta_d \cdot \lfloor (N - i - 1) / 2 \rfloor $
      \Else
	\State $k \gets \lfloor (N - i) / 2 \rfloor$ 
	\State $\Delta_m \gets k \cdot \Delta_d + 2 \cdot (k-1) \cdot \Delta_d$
      \EndIf
    \EndIf

    \State \textbf{return} $\Delta_m$ 
   \EndProcedure
  \end{algorithmic} 
\end{algorithm}

\begin{figure}[h]
 \centering
 \includegraphics[bb=0 0 150 173,scale=1]{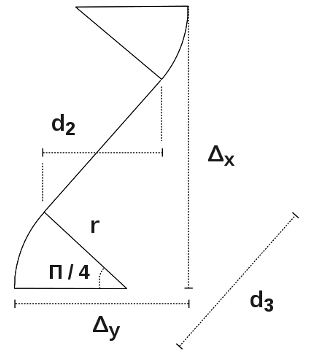}
 \caption{An aircraft goes to the middle of the squadron
when it assumes the longitudinal formation. }
 \label{fig:geometria-trajetoria}
\end{figure}

The interface of the transition algorithm is FLATLO($i$, $\Delta_d$, $vel$, $r$, $\varphi$, $\beta$, $P$), 
where $i$ is the aircraft index shown in Figures \ref{fig:latitudinal-formation} and \ref{fig:longitudinal-formation}, 
$vel$ is the aircraft airspeed, $r$ is the radius of the  \textit{to-turn}  maneuvers
used by the transition algorithm, $\varphi$ and $\beta$ are shown in Figure \ref{fig:to-turn}a, and
$P$ is the aircraft initial position $(x_0, y_0, h_0)$. 

Algorithm \ref{algo:fltl} is executed by each aircraft processor  in parallel.  For all aircraft, except the last one, the forward displacement, called $\Delta_m$, is calculated at Line \ref{linha:deslocamento}. See
Algorithm \ref{algo:advancing} for details. 
Thereafter, each aircraft executes four maneuvers. The values that are determined between
Lines \ref{linha:inicio_encaixe} and \ref{linha:final_encaixe} are used to fit each maneuver with 
the next one. 

Each aircraft moves toward the longitudinal line of the squadron between Lines
\ref{linha:inicio_primeira_curva} and \ref{linha:final_primeira_curva}. At Line \ref{linha:vai_em_direcao_ao_centro}, 
the aircraft goes forward toward the longitudinal line. Between Lines \ref{linha:inicio_segunda_curva}
and \ref{linha:final_segunda_curva}, the aircraft moves to smoothly enter the longitudinal line.
Finally, the aircraft flies over the longitudinal line at Line \ref{linha:manobra_final}. 

The references created by Algorithm \ref{algo:fltl} were tested by using the aircraft model presented in Subsection \ref{subsection:model}.

\subsection{The Aircraft Model}
\label{subsection:model}

The simple and well-tested aircraft state space model, \cite{Anderson}, is employed  and
is given by

\begin{equation}
 \label{eq:velocidade}
 \frac{d V}{d t} = g \cdot \left[ \frac{(T - D)}{W} - sin(\gamma) \right]
\end{equation} 
\begin{equation}
 \label{eq:gama}
 \frac{d \gamma}{d t} = \frac{g}{V} \cdot [ n \cdot cos(\mu) - cos(\gamma)]
\end{equation} 
\begin{equation}
 \label{eq:heading}
 \frac{d \chi}{d t} = \frac{g \cdot n \cdot sin(\mu)}{V \cdot cos(\gamma)}
\end{equation} 
\begin{equation}
 \label{eq:x}
 \frac{d x}{d t} = V \cdot cos(\gamma) \cdot cos(\chi)
\end{equation} 
\begin{equation}
 \label{eq:y}
 \frac{d y}{d t} = V \cdot cos(\gamma) \cdot sin(\chi)
\end{equation} 
\begin{equation}
 \label{eq:z}
 \frac{d h}{d t} = V \cdot sin(\gamma)
\end{equation} 
where the state variables are: airspeed (V), flight path angle ($\gamma$), flight path heading ($\chi$), 
and the position variables (x, y, h). 

A control scheme is presented in \cite{Giacomin-cobem} by using the above aircraft model.
Here, the references created by Algorithm \ref{algo:fltl} are submitted to this control scheme. 

\section{Simulation Results}
\label{sec:simulation}

Algorithm \ref{algo:fltl} is programmed in parallel by using the C++ programming language
and the GNU Message Passing Interface (MPI) Compiler. One processor is allocated for each aircraft.
The graphics are plotted by using the software Gnuplot. 

The initial condition for all aircraft is: height: $3,050.00$ meters, airspeed:
$30.5$ m/s, $\Delta_d = 18,300.00$ meters, assuming the latitudinal formation. 
The references are created by Algorithm \ref{algo:fltl} by using a radius of $4,575.00$ meters  and  different airspeeds for each aircraft,
which allow all aircraft to arrive in longitudinal formation at the same  instant.
The simulation is done by using the Runge-Kutta-4
algorithm. The results are shown in Figures \ref{fig:squadron-references}, \ref{fig:squadron-trajectories} and
\ref{fig:velocity-dinamic}.

\begin{figure}[h]
 \centering
  \includegraphics[bb=0 0 220 134,scale=1]{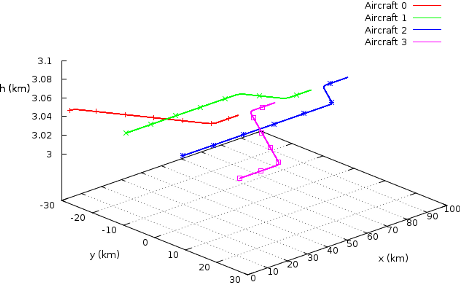}
 \caption{The aircraft references. The marks are used in the aircraft crossing analysis.}
 \label{fig:squadron-references}
\end{figure}

\begin{figure}[h]
 \centering
 \includegraphics[bb=0 0 220 135,scale=1]{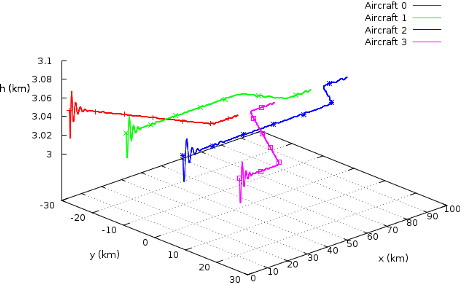}
 \caption{The aircraft trajectories. The aircraft succeeded in following the references. }
 \label{fig:squadron-trajectories}
\end{figure}

The marks shown in Figures \ref{fig:squadron-references} and \ref{fig:squadron-trajectories} are
used to analyze the aircraft crossing. 
When the aircraft number zero is flying over the third mark, the aircraft 
number one is ahead, and when aircraft number three is flying over the fourth mark, the aircraft number two is ahead. An automatic verification  showed that the minimum distance between aircraft $0$ and $1$ and
between aircraft $2$ and $3$ was $14.5$ km and $18.3$ km, respectively. 
Therefore, there is no collision. 

The simulation is done by considering a noise of $0.25\%$ for airspeed,
heading and gamma angles. From Figures \ref{fig:squadron-references}, \ref{fig:squadron-trajectories}, and \ref{fig:velocity-dinamic} it is possible to conclude that Algorithm \ref{algo:fltl}
creates the references correctly and that the aircraft succeeded in following the references.

\begin{figure}[h]
 \centering
 \includegraphics[bb=0 0 211 139,scale=1]{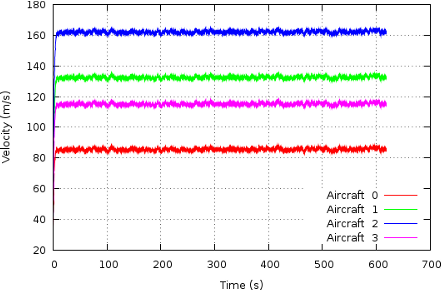}
 \caption{All aircraft velocities. }
 \label{fig:velocity-dinamic}
\end{figure}

It remains to analyze the time complexity of Algorithm \ref{algo:fltl}
and to prove its  correctness.  This is done in Section \ref{sec:theoretical-analysis}.

\section{Theoretical Analysis}
\label{sec:theoretical-analysis}

The theoretical analysis of Algorithm \ref{algo:fltl} is divided into two parts: its time 
complexity analysis and its proof of correctness.

\subsection{The Time Complexity Analysis}

The  \textit{go-forward}  and  \textit{to-turn}  functions are executed over a fixed
number of maneuvers, and all other functions used in Algorithm \ref{algo:fltl}
have constant time complexity. 
Let $Q = \{q_0, \cdots, q_{|Q|-1}\}$ be the 
set of maneuvers, and let $R = \{r_0, \cdots, r_{|Q|-1}\}$ be the
set of references, with $r_i = ref(q_i)$, and the function \textit{ref}
be implemented by using the functions FW or C in Algorithm \ref{algo:fltl}.
Then, $R$ is a set of sets, and $|r_i|$ is the number of references
for each maneuver $q_i$. The time complexity of Algorithm \ref{algo:fltl} is

\[O(|r_0| + |r_1| + |r_2| + |r_3| + |r_4|)\]
\[m = |r_0| + \cdots + |r_4| \]
\[O(m)\]
where $m$ represents the total  number of  time steps for each aircraft trajectory.
\textit{Since this is also the lower bound for the problem, it is concluded that the FLATLO
algorithm is optimal with regard to the time complexity \cite{Cormen}}. 

\subsection{The Proof of Correction}

The proof of correctness takes into account that the maneuvers executed by
Algorithm \ref{algo:fltl} are basically described in Figure \ref{fig:geometria-trajetoria}, 
except the first one, which is created by Algorithm \ref{algo:advancing}.

Note that when an aircraft in latitudinal formation advances the distance $\Delta_d$
before executing the trajectory of Figure \ref{fig:geometria-trajetoria}, 
$\Delta_x$ advances by $\Delta_d$ along its longitudinal formation.
This information is used in the proof of the next theorem. 

\begin{theorem}
\label{theorem:correcao}
Let  us  assume a set of aircraft initially flying in latitudinal formation, with
each aircraft equally spaced by $\Delta_d$ from its neighbors. 
Then, if the FLATLO algorithm is executed by each aircraft, 
the longitudinal formation, with spacing $\Delta_d$ between its neighbors, is achieved. 
\end{theorem}
 
\begin{proof}
It is considered $n$ aircraft in the squadron.
Clearly, if $n$ is even and $i$ is even, it can be concluded by 
analyzing Algorithm \ref{algo:advancing} that $\Delta_m^i - \Delta_m^{i+1} = \Delta_d$. 
Then

\[\Delta_{x}^i = k \cdot \Delta_d + 2 \cdot r \cdot sin(\pi/4) + \Delta_y - 2 \cdot r \cdot (1 -cos(\pi / 4))\]
\[\Delta_{x}^{i+1} = (k-1) \cdot \Delta_d  + 2 \cdot r \cdot sin(\pi/4) + \Delta_y - \cdots \]
\[\cdots -  2 \cdot r \cdot (1 -cos(\pi / 4))\]
\[\Delta_{x}^i - \Delta_{x}^{i+1} = \Delta_d\]
 which  is the obvious case where the two aircraft $i$ and $i+1$ are mirrored by the middle line. If $i$ is odd,
it follows, by analyzing Algorithm \ref{algo:advancing}, that $\Delta_m^i - \Delta_m^{i+1} = 2 \cdot \Delta_d$. Then

\[\Delta_{x}^i = k \cdot \Delta_d + 2 \cdot r \cdot sin(\pi/4) + \Delta_y - 2 \cdot r \cdot (1 -cos(\pi / 4))\]
\[\Delta_{x}^{i+1} = (k-2) \cdot \Delta_d + 2 \cdot r \cdot sin(\pi/4) + (\Delta_d + \Delta_y) - \cdots \]
\[ \cdots -  2 \cdot r \cdot (1 -cos(\pi / 4))\]
\[\Delta_{x}^i - \Delta_{x}^{i+1} = k \cdot \Delta_d - (k-2) \cdot \Delta_d - \Delta_d\]
\[\Delta_{x}^i - \Delta_{x}^{i+1} = \Delta_d\]

On the other hand, if $n$ is odd and $i$ is even, it follows, by analyzing Algorithm \ref{algo:advancing}, that
$\Delta_m^i - \Delta_m^{i+1} = 2 \cdot \Delta_d$,  which is  the same case that happens when $n$ is even and $i$
is odd. If $n$ is odd and $i$ is odd, then it follows, by analyzing Algorithm \ref{algo:advancing}, that
$\Delta_m^i - \Delta_m^{i+1} = \Delta_d$, which has the same result as that obtained when $n$ is even and $i$ is even.
Therefore, $\Delta_{x}^i - \Delta_{x}^{i+1} = \Delta_d$ for every $n$ and for every $i$. 

Similar reasoning can be applied to $\Delta_y$. Lines \ref{linha:delta_y_comeco} to \ref{linha:delta_y_final}
show that $\Delta^{i+2} - \Delta^i = \Delta_d$, for every $n$ and every $i$. Then, if $n$ is even

\[\Delta_{y}^i = 2 \cdot r \cdot (1 -cos(\pi / 4)) + \Delta_d / 2 + k \cdot \Delta_d + \Delta_x - \cdots \]
\[ \cdots - 2 \cdot r \cdot sin(\pi/4)\]
\[\Delta_{y}^{i+2} = 2 \cdot r \cdot (1 -cos(\pi / 4)) + \Delta_d / 2 + \Delta_d +  \cdots \]
\[ \cdots + (k-1) \cdot \Delta_d + \Delta_x - 2 \cdot r \cdot sin(\pi/4)\]
\[\Delta_y^{i+2} - \Delta_y^i = 0\]
and if $n$  is  odd it follows that

\[\Delta_{y}^i = 2 \cdot r \cdot (1 -cos(\pi / 4)) + k \cdot \Delta_d + \Delta_x - \cdots \]
\[ \cdots - 2 \cdot r \cdot sin(\pi/4)\]
\[\Delta_{y}^{i+2} = 2 \cdot r \cdot (1 -cos(\pi / 4)) + \Delta_d + (k-1) \cdot \Delta_d + \cdots \]
\[ \cdots + \Delta_x - 2 \cdot r \cdot sin(\pi/4)\]
\[\Delta_y^{i+2} - \Delta_y^i = 0\]
and therefore, $\Delta_y^{i+2} - \Delta_y^i = 0$ for every $n$ and for every $i$. 
Since this same reasoning can be applied for other coordinate system bases, and since
all aircraft arrive in longitudinal formation at the same instant,
the proof follows.
\end{proof}

\section{Conclusion}
\label{sec:conclusion}

An algorithm for changing from latitudinal to
longitudinal formation for autonomous aircraft squadrons is proposed in this paper. Despite the relevance of this problem, an extensive literature review did not produce relevant results.
 
The proposed FLATLO algorithm's time complexity is equal to the problem lower  bound; hence,  it is optimal  \cite{Cormen}.

It was proved that if the squadron is initially in latitudinal
formation, with each aircraft equally spaced from its neighbors
by distance $\Delta_d$, then the proposed algorithm makes the
squadron change its formation to the longitudinal one,
keeping the same distance $\Delta_d$ from each aircraft
and its neighbors. 

The theoretical analysis was confirmed by the simulations showing that
the aircraft do not collide during the formation transition, and
that the references were created correctly by the proposed algorithm, 
such that they could be followed by a control scheme. 
Additionally, since the aircraft operate at different velocities, the number of aircraft has to be limited, and the velocities have to be verified at design time, for safety reasons.

\end{document}